\documentclass[twoside]{article}

\usepackage{hyperref} % put this in your preamble

\usepackage[ruled,vlined]{algorithm2e}
\usepackage[authoryear]{natbib}
\usepackage[preprint]{aistats2026}
\usepackage{microtype}
\usepackage{graphicx}
\usepackage{subfigure}
\usepackage{booktabs} % for professional tables
\usepackage{multirow} 

\usepackage{amsmath}
\usepackage{amssymb}
\usepackage{mathtools}
\usepackage{amsthm}

\usepackage[capitalize,noabbrev]{cleveref}

\theoremstyle{plain}
\newtheorem{theorem}{Theorem}[section]

\newtheorem{lemma}[theorem]{Lemma}

\theoremstyle{definition}
\newtheorem{definition}[theorem]{Definition}

\theoremstyle{remark}

\newcommand\R{\mathbb{R}}

\usepackage[most]{tcolorbox} 

\definecolor{lightgreen}{RGB}{220, 255, 220}
\definecolor{lightpink}{RGB}{255, 220, 230}

\usepackage{tikz}
\usetikzlibrary{arrows.meta,positioning,calc}
\usepackage{subcaption} % for subfigure
\usepackage[framemethod=TikZ]{mdframed}
\usepackage[framemethod=TikZ]{mdframed}
\definecolor{lightGray}{gray}{0.97}
\definecolor{midGray}{gray}{0.40}
\definecolor{lightYellow}{RGB}{254,254,239}
\definecolor{darkGray}{rgb}{0.45,0.45,0.45}
\definecolor{darkerGray}{rgb}{0.3,0.3,0.3}
\definecolor{purpleblue}{RGB}{90,96,136}
\mdfdefinestyle{myFrameStyle}{%
  linecolor=purpleblue,
  linewidth=0.7pt,
  roundcorner=3pt,
  skipabove=5pt,
  skipbelow=0pt,
  innertopmargin=5pt,
  innerbottommargin=5pt,
  innerrightmargin=5pt,
  innerleftmargin=5pt,
  leftmargin=0pt,
  rightmargin=0pt,
}

\begin{document}

% If your paper is accepted and the title of your paper is very long,
% the style will print as headings an error message. Use the following
% command to supply a shorter title of your paper so that it can be
% used as headings.
%
%\runningtitle{I use this title instead because the last one was very long}

% If your paper is accepted and the number of authors is large, the
% style will print as headings an error message. Use the following
% command to supply a shorter version of the author names so that
% they can be used as headings (for example, use only the surnames)
%
%\runningauthor{Surname 1, Surname 2, Surname 3, ...., Surname n}

\twocolumn[

\aistatstitle{Preconditioned Attention: Enhancing Efficiency in Transformers}

\aistatsauthor{Hemanth Saratchandran}

\aistatsaddress{Australian Institute for Machine Learning (AIML), Adelaide University \\
\texttt{hemanth.saratchandran@adelaide.edu.au}}]

\begin{abstract}
Central to the success of Transformers is the attention block, which effectively models global dependencies among input tokens associated to a dataset. However, we theoretically demonstrate that standard attention mechanisms in transformers often produce ill-conditioned matrices with large condition numbers. This ill-conditioning is a well-known obstacle for gradient-based optimizers, leading to inefficient training. To address this issue, we introduce preconditioned attention, a novel approach that incorporates a conditioning matrix into each attention head. Our theoretical analysis shows that this method significantly reduces the condition number of attention matrices, resulting in better-conditioned matrices that improve optimization. Conditioned attention serves as a simple drop-in replacement for a wide variety of attention mechanisms in the literature. We validate the effectiveness of preconditioned attention across a diverse set of transformer applications, including image classification, object detection, instance segmentation, long sequence modeling and language modeling.
\end{abstract}

\section{Introduction}

\begin{figure*}[ht!]
    \centering
    \includegraphics[width=1.0\linewidth]
    {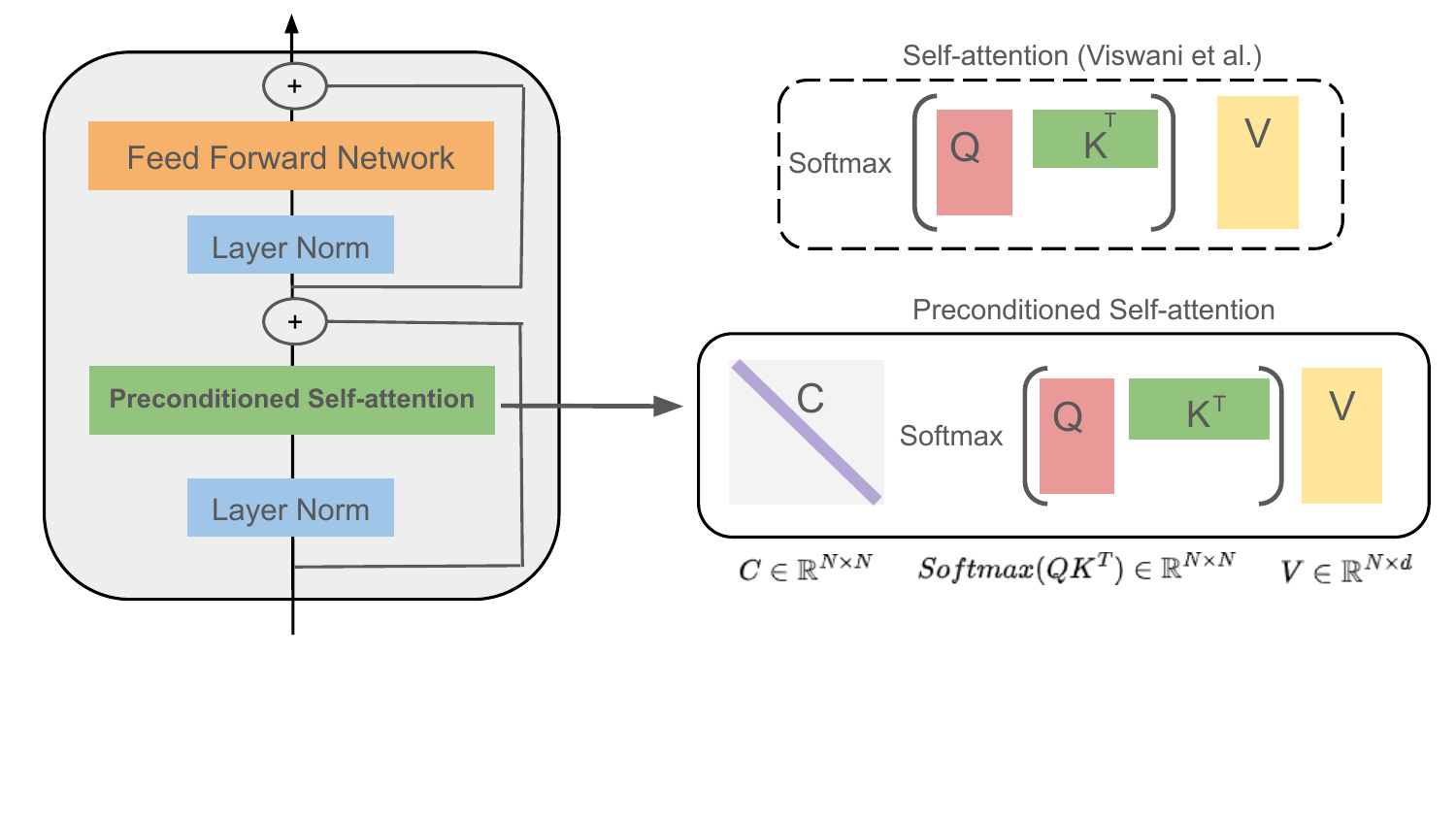}
    %\caption{Trained for 350s}
    %\end{subfigure}
    \vspace{-8em}
    \caption{Schematic representation of preconditioned self-attention. Left: A layer of a general transformer employing a preconditioned self-attention block. Right: Self-attention (top) and preconditioned self-attention (bottom) are compared. The key difference is that preconditioned self-attention applies a multiplication by a diagonal preconditioner matrix C, which depends on the query $Q$, key $K$, and value $V$. }
    \label{fig:front_fig}
\end{figure*}

Transformers have revolutionized machine learning, emerging as the dominant architecture for a diverse range of applications, including natural language processing, computer vision, and robotics. Introduced by \cite{vaswani2017attention}, transformers have redefined how models process and interpret data by capturing both local and global dependencies. Central to this innovation is the self-attention mechanism, a key component that computes pairwise interactions among all input tokens simultaneously. By dynamically weighing the relative importance of tokens, self-attention enables transformers to model intricate patterns and complex relationships within data, establishing them as a foundational architecture for modern machine learning applications.

In this paper, we focus on the conditioning of the self-attention matrix. The conditioning of a matrix is measured by its condition number, defined as the ratio of its largest to smallest singular value. It is well-established that poorly conditioned matrices (i.e. those with high condition numbers) can adversely impact the convergence of gradient-based optimization algorithms, making training challenging \citep{nocedal1999numerical}. While prior studies have examined the conditioning of feedforward layers in neural networks \citep{agarwal2021deep, liu2022loss, saratchandran2025weight}, the conditioning of self-attention mechanisms in transformers remains largely unexplored.

We introduce a theoretical framework to analyze the condition number of the self-attention matrix and show that it can be bounded by a quantity dependent on the Frobenius norm of the values matrix comprising the self-attention mechanism. Leveraging this insight, we propose a method to reduce the condition number of the self-attention matrix. Drawing inspiration from numerical methods, we employ a preconditioner—an auxiliary matrix designed to lower the conditioning of the target matrix. However, constructing an effective preconditioner is often challenging, as it requires tailoring to the specific structure of the problem. Leveraging the unique structure of the self-attention matrix, we propose a novel preconditioner that significantly reduces its condition number. This approach, which we term preconditioned attention, improves the conditioning of the self-attention mechanism and lays the groundwork for more stable and efficient training of transformers. \cref{fig:front_fig} presents a schematic representation of a preconditioned attention layer. In this design, the self-attention matrix is multiplied on the left by a square diagonal matrix $C$, which serves to condition the self-attention matrix, effectively reducing its condition number. The matrix $C$ is dynamically computed based on the queries $Q$, keys $K$, and values $V$ of the self-attention matrix.

We demonstrate the effectiveness of preconditioned self-attention across a diverse range of applications, including image classification, object detection, instance segmentation, text classification, and physics-based modeling. In each case, we show that preconditioned self-attention serves as a robust and efficient replacement for various attention mechanisms commonly used in the literature, consistently achieving superior results across all these tasks. One of the key strengths of preconditioned self-attention lies in its versatility. It is agnostic to the underlying attention mechanism, making it compatible with a wide range of architectures. As we demonstrate, it can serve as a drop in replacement to advanced attention mechanisms proposed in recent works \citep{ali2021xcit, liu2021swin, touvron2021training, yuan2022volo, ding2022davit, xiong2021nystromformer}, offering a straightforward yet powerful enhancement to existing models. Our main contributions are:
\begin{itemize}
    \item[1.] A theoretical framework demonstrating that self-attention mechanisms are inherently ill-conditioned, along with the introduction of preconditioned self-attention—a novel approach that employs a preconditioner to modify the self-attention matrix and significantly reduce its condition number.
    \item[2.] An empirical evaluation of preconditioned attention across various transformer architectures and diverse applications from the literature, demonstrating the effectiveness of our proposed methodology.
\end{itemize}

\section{Related Work}\label{sec:rel_work}

\paragraph{Efficient Attention.} In recent years, numerous methods have been proposed to improve the efficiency and effectiveness of transformers, particularly by addressing their computational complexity and rethinking attention mechanisms. The Data-efficient Image Transformer (DeiT) \citep{touvron2021training} focuses on reducing data requirements by utilizing distillation tokens, which significantly enhance the training efficiency and performance of vision transformers without relying on large-scale datasets. Cross-Covariance Image Transformer (XCiT) \citep{ali2021xcit} introduces a novel approach to attention, operating on the cross-covariances of spatial features, thereby enabling efficient interactions across spatial dimensions and reducing computational costs. Vision Outlooker (VOLO) \citep{yuan2022volo} refines the attention mechanism by incorporating an outlook attention module, designed to capture long-range dependencies while maintaining computational efficiency, leading to superior performance in vision tasks. Meanwhile, Nyströmformer \citep{xiong2021nystromformer} employs a Nyström-based approximation for the standard self-attention mechanism, reducing its quadratic complexity to near-linear time while retaining the core properties of attention. This transformer has shown great results in modeling long range sequence modeling. We will show that our framework of conditioning self-attention can be added as a drop in replacement to many of these newer forms of attention yielding superior results.

\paragraph{Conditioning.}
A long line of research has examined neural network training through the perspective of conditioning. Early analyses framed the problem in terms of the neural tangent kernel (NTK), showing that its spectral properties largely determine optimization behavior in the infinite-width limit. In particular, well-conditioned NTKs were linked to more stable convergence~\citep{jacot2018neural, liu2022loss}. Beyond kernel analyses, architectural choices have also been identified as critical for conditioning \cite{saratchandran2024rethinking, zheng2025structured, macdonald2023skip, chng2025preconditioners, saratchandran2025leaner, saratchandran2025enhancing, saratchandran2026spectral, qin2022cosformer}. Depth, for instance, has been shown to enhance the conditioning of gradient updates, facilitating more effective optimization~\citep{agarwal2021deep}. In the context of Transformers, skip connections have been highlighted as an intrinsic mechanism that regularizes conditioning and prevents instabilities in very deep networks~\citep{ji2025always, ji2025cutting}. Furthermore, works in fine-tuning have shown how to better the conditioning of low-rank adapters for better training \cite{ji2024sine, albert2025randlora, albert2025towards}.
%In contrast to these perspectives, we provide a methodology that directly conditions the attention layer in a transformer leading to better convergence. 

%on whether conditioning can be embedded directly at initialization, aiming to design attention layers that are well-behaved from the very start of training.

\section{Notation}\label{sec:prelims}

In this section, we formally define the transformer architecture by describing the structure of a transformer layer, along with notation for various mathematical elements that will be referenced in subsequent sections. For further foundational details on transformers, readers may consult \citet{vaswani2017attention}.

A layer in a transformer can be represented as a mapping 
\begin{equation}
  \mathbf{T}: \mathbb{R}^{N \times D} \rightarrow \mathbb{R}^{N \times D}  
\end{equation}
defined by 
\begin{equation}\label{eqn:trans_main}
    \mathbf{T}(X) = \mathbf{F}(\mathbf{A}(X) + X),
\end{equation}
where \( \mathbf{F} \) denotes a feed forward network with a residual connection, and \( \mathbf{A} \) represents an attention mechanism.

Self-attention is composed of three learnable matrices—query (\( Q \)), key (\( K \)), and value (\( V \))—defined for an input sequence \( X \in \mathbb{R}^{N \times D} \) as follows: \( q = XQ \), \( k = XK \), and \( v = XV \), where \( Q, K \in \mathbb{R}^{D \times d} \) and \( V \in \mathbb{R}^{D \times d} \). The output of the attention head \( \mathbf{A}(X) \) is then given by
\begin{equation}\label{eqn:attn_eqn_general}
    \mathbf{A}(X) = \mathbf{softmax}(qk^T)v,
\end{equation}
where $\mathbf{softmax}$ is the softmax activation that acts row-wise on the matrix $qk^T$. Note that then $A(X) \in \R^{N\times d}$.

In general, multiple attention heads \( \mathbf{A}_i \) for \( 1 \leq i \leq h \) are utilized, each of dimension $N \times d_i 
= N \times \frac{d}{h}$.
These are then concatenated to produce a multi-head attention output,
\begin{equation}\label{eqn:multi_head_eqn}
    \mathbf{A} = [\mathbf{A}_1, \cdots, \mathbf{A}_h],
\end{equation}
which is subsequently fed into the feedforward layer. The full transformer architecture is obtained by sequentially stacking several such transformer layers, as defined in \cref{eqn:trans_main}.

Given an arbitrary matrix $M$ we will denote the Frobenius norm of $M$ by $\vert\vert M\vert\vert_F$. 
If $M$ is of size $n \times d$ and $k = \min\{n, d\}$ then we will denote the singular values of $M$ by $\sigma_1,\ldots, \sigma_k$ and order them so that 
$\sigma_1 \geq \cdots \geq \sigma_k$. We remind the reader that the following formula exists 
$\vert\vert M\vert\vert_F = \sqrt{\sigma_1^2+\cdots + \sigma_k^2}$. We will use the standard terminology SVD to denote the singular value decomposition of a matrix. In our main theorem we will need to consider the rows of a matrix $M$. We will denote the i-th row of $M$ by $M_{i,}$, where if $M$ has size $n \times d$ then each
$M_{i,}$ is a vector of shape $1 \times d$ for each $1 \leq i \leq n$. The 2-norm of this vector will be denoted by 
$||M_{i,}||_2$.

\section{Theoretical Framework}

\subsection{Conditioning of self-attention}

In this section we study the conditioning of the self-attention matrix of a transformer. 

\begin{definition}
Let $A$ be an $n \times d$ matrix of full rank. The condition number of $A$, denoted by $\kappa$, is defined as
\begin{equation}
    \kappa(A) = \frac{\sigma_1}{\sigma_k}
\end{equation}
where $k = \min\{n, d\}$, and $\sigma_1$ denotes the largest singular value of $A$ and $\sigma_k$ the smallest singular value of $A$. Note that as $A$ is assumed full rank $\sigma_k \neq 0$ and the condition number is well-defined. Furthermore, observe that $\kappa(A) \geq 1$ since 
$\sigma_1 \geq \sigma_k$. 
\end{definition}

The condition number plays a crucial role in iterative algorithms, such as gradient descent, as matrices with lower condition numbers within a non-linear system lead to better convergence of these algorithms \cite{nocedal1999numerical, agarwal2021deep, liu2022loss, saratchandran2025weight}.

In general, for large matrices computing the singular value decomposition of the matrix each time to compute the condition number is extremely difficult. The following estimate of 
Guggenheimer et al. \cite{guggenheimer1995simple} gives a useful bound on the condition number of a full rank matrix which can serve as a good measure of how complex the condition number can be.

\begin{theorem}[\cite{guggenheimer1995simple}]\label{thm:guggen}
 Let $A$ be an $n \times d$ matrix of full rank. Let 
 $k = \min\{n ,d\}$.
 Then the condition number of $A$ has the following bound
\begin{equation}\label{eqn:guggen}
    \kappa(A) \leq \frac{2}{\sigma_1\cdots\sigma_k}\bigg{(}
    \frac{\vert\vert A\vert\vert_F}{\sqrt{k}}
    \bigg{)}^k.
\end{equation}
\end{theorem}

Our first main insight is that the above theorem gives a bound on the condition number of self-attention that depends on the values matrix $v$ of the self-attention $\mathbf{softmax}(qk^t)v$. 

\begin{theorem}\label{thm:condition_self_attn}
    Let $A = \mathbf{softmax}(qk^T)v$ be the standard self-attention of size $n \times d$ and assume $A$ is of full rank. Let $\sigma_1 \geq \cdots \geq \sigma_k$ denote the singular values of $A$, where $k = \min\{n, d\}$. Then
\begin{equation}
\kappa(A) \leq  \frac{2}{\sigma_1\cdots\sigma_k}\bigg{(}
\frac{\sqrt{n}}{\sqrt{k}}\vert\vert v\vert\vert_F\bigg{)}^k
\end{equation}
\end{theorem}

The proof of \cref{thm:condition_self_attn} is given in \cref{app:theory}. We observe that if $\vert\vert v\vert\vert_f > 1$ then since $n \geq k$ the term $\bigg{(}
\frac{\sqrt{n}}{\sqrt{k}}\vert\vert v\vert\vert_F\bigg{)}^k$ will grow exponentially in $k$. Thus for large attention matrices of full rank this term can be extremely large leading to a high condition number.

\subsection{A preconditioner for self-attention}\label{subsec:precond}

In \cref{thm:condition_self_attn} we saw that the self-attention matrix in a transformer layer can be bounded by a quantity that depended exponentially on the Frobenius norm of the values matrix. In this section, we would like to design a matrix $C$ so that when we consider the product
\begin{equation}
    C\cdot \mathbf{softmax}(qk^T)v
\end{equation}
it satisfies
\begin{equation}
    \kappa(C\cdot \mathbf{softmax}(qk^T)v) \leq 
    \kappa( \mathbf{softmax}(qk^T)v).
\end{equation}
In other words, we want to understand whether there is a matrix $C$ that we can use to multiply the self-attention matrix
$\mathbf{softmax}(qk^T)v$ with so as to reduce the condition number. Such a $C$ is called a preconditioner.

In general, computing the condition number of a matrix directly via the SVD is a costly process. However, \cref{thm:guggen} provides a useful upper bound that can serve as a complexity measure of the condition number of a matrix without needing to directly access the SVD. We define this complexity measure of a matrix $A$ with rank $k$ by
\begin{equation}\label{eqn:condition_complxity_quant}
\mu(A) :=  \frac{2}{\sigma_1\cdots\sigma_k}\bigg{(}
    \frac{\vert\vert A\vert\vert_F}{\sqrt{k}}
    \bigg{)}^k.  
\end{equation}
Our goal will be to reduce the quantity $\mu(A)$. The following theorem produces a matrix $C$ that when multiplied with a self-attention matrix brings down the quantity $\mu(A)$.

\begin{theorem}\label{thm:precond_attn}
Let $A = \mathbf{softmax}(qk^T)v$ denote a self-attention matrix, of size $n \times d$, within a transformer layer and assume $A$ has full rank. Let 
$C$ be defined as the $n \times n$ diagonal matrix whose i-th diagonal entry is given by $\frac{1}{||A_{i,}||_2}$. Then
\begin{equation}
    \mu(C\cdot A) \leq \mu(A).
\end{equation}
\end{theorem}

The proof of \cref{thm:precond_attn} is given in \cref{app:theory}. \cref{thm:precond_attn} motivates the consideration of a new type of attention mechanism defined by the equation
\begin{equation}
    C\cdot \mathbf{softmax}(qk^T)v
\end{equation}
where $C$ is the matrix given by \cref{thm:precond_attn}. For multi-head attention we condition each head thereby forming
\begin{equation}
 [C_1\cdot \mathbf{softmax}(q_1k_1^T)v_1,\ldots, 
 C_h\cdot \mathbf{softmax}(q_hk_h^T)v_h].   
\end{equation}
We call such an attention mechanism preconditioned self-attention.
Note that although the theory presented in this section focuses on an upper bound of self-attention our main insight, as show in \cref{sec:experiments}, is that this still serves as a good way to motivate the need to condition self-attention. Furthermore, the theoretical framework focused on self-attention, however 
in \cref{sec:experiments} we will show that the preconditioner matrix obtained in \cref{thm:precond_attn} works well with other forms of attention. In \cref{sec:experiments} we plot condition numbers of transformers.

\begin{algorithm}[t]
\caption{Forward Pass: Preconditioned Attention}
\label{alg:precond_attn}
\KwIn{Queries $q$, Keys $k$, Values $v$}
\KwOut{Output representation $y$}
\begin{enumerate}
    \item Compute attention weights: $A \gets \mathbf{softmax}(qk^\top)$.
    \item Construct the preconditioner $C$ as defined in \cref{thm:precond_attn}.
    \item Apply the preconditioner: $\tilde{A} \gets C \cdot A$.
    \item Compute the output: $y \gets \tilde{A} \cdot v$.
\end{enumerate}
\end{algorithm}

\Cref{alg:precond_attn} illustrates how preconditioned attention is incorporated into the forward pass of a transformer. Importantly, the preconditioner \(C\) is treated as a fixed transformation: it is not differentiable, carries no gradients, and therefore plays no role in the backward pass of optimization.

\subsection{Computational cost}\label{subsec:cost}

In the previous section we introduced preconditioned self-attention which takes the form
\begin{equation}
C\cdot \mathbf{softmax}(qk^T)v    
\end{equation}
where $C$ is a diagonal matrix, known as a preconditioner, that is built from the row norms of the self-attention matrix
$\mathbf{softmax}(qk^T)v$. 

Observe that since $C$ depends on $q$, $k$ and $v$, and these matrices are learnable, when training a transformer using preconditioned self-attention we have to recompute $C$ during each forward pass. This will necessarily add some computational overhead to the forward pass (see \cref{alg:precond_attn}). In this section, we will compute the extra floating point operations (FLOPs) for the forward pass.

Let $A = \mathbf{softmax}(qk^T)v$ denote one head of the self-attention block in one layer of a transformer and let us assume that $A$ has size $n \times \frac{d}{h}$, where $h$ is the total number of heads. The preconditioner matrix $C$ given by \cref{thm:precond_attn} is computed as a diagonal matrix where the
i-th diagonal element $C_{ii}$ of $C$ is given by the inverse of the 2-norm of the i-th row of $A$ namely.

To compute the number of FLOPs involved in computing $C$ through each forward pass we break down the computation.

\paragraph{Step 1: Compute the 2-norm of each row of $A$.} 
For each row $A_{i,}$ of $A$, where $i = 1, 2, \ldots, n$, we have:
\begin{enumerate}
    \item \textbf{Square each element of the row:} This requires $\frac{d}{h}$ multiplications.
    \item \textbf{Sum the squared elements:} This requires $\frac{d}{h}-1$ additions.
    \item \textbf{Take the square root of the sum:} This requires 1 operation.
\end{enumerate}
Thus, for a single row, the total operations are:
\begin{equation}\label{eqn:single_row_op}
\frac{d}{h} + \left(\frac{d}{h} - 1\right) + 1 = 2\frac{d}{h}.
\end{equation}
For $n$ rows, the total operations to compute the 2-norms are:
\begin{equation}\label{eqn:n_row_op}
 n \cdot 2\frac{d}{h} = 2n\frac{d}{h}.   
\end{equation}

%\begin{table*}[!ht]
% \caption{Comparison of vision transformers with their original attention mechanism verse one with a preconditioned attention mechanism pre-trained on the ImageNet-1k dataset. We report the classification top-1\% accuracy.}
%    \centering
%    \begin{tabular}{c|c|c|c|c|c}

%        \midrule
%        \multirow{2}{*}{} & \multicolumn{4}{c}{Models} \\
%        & ViT-base & DeiT-base & Swin-base & XciT-medium & DaViT-base\\
%        \midrule
%        Original & 80.2 & 81.6 & 83.4 & 82.6 & 84.2 \\
%        \midrule
%       Preconditioned & \textbf{81.4} & \textbf{82.5} & \textbf{84.5} & \textbf{83.4} & \textbf{84.9} \\
%        \midrule
%     \end{tabular}
%    \label{tab:vits}
%\end{table*}

\begin{table*}[!ht]
 \caption{Comparison of vision transformers with their original attention mechanism versus one with a preconditioned attention mechanism pretrained on the ImageNet-1k dataset. We report the classification top-1\% accuracy.}
    \centering
    \begin{tabular}{c|c|c|c|c}
        \midrule
        \multirow{2}{*}{} & \multicolumn{4}{c}{Models} \\
        & ViT-base & XciT-medium & DeiT-base & Swin-base \\
        \midrule
        Original & 80.2 & 82.6 & 81.6 & 83.4 \\
        \midrule
        Preconditioned & \textbf{81.4} & \textbf{83.5} & \textbf{82.7} & \textbf{84.5} \\
        \midrule
     \end{tabular}
    \label{tab:vits}
\end{table*}

\paragraph{Step 2: Compute the inverse of each 2-norm.}
For each of the $n$ rows, we then need to compute the reciprocal of the 2-norm:
\begin{equation}\label{eqn:inverse_row_norm}
C_{ii} = \frac{1}{||A_{i,}||_2}.    
\end{equation}
This requires $1$ division per row, for a total of:
\begin{equation}
n \text{ (divisions)}.    
\end{equation}

\paragraph{Total operations:} Summing up the operations from step 1 and step 2 we get that the total number of FLOPs is
\begin{equation}\label{eqn:flops_for_C}
  \text{ Total FLOPs for computing C } = n(2\frac{d}{h} + 1).   
\end{equation}
As there are $h$ heads per layer we then get that the total number of extra FLOPs for preconditioned attention per layer is
\begin{equation}\label{eqn:total_flops}
\text{ Total extra FLOPs per layer } = n(2d + h)     
\end{equation}
which is obtained by summing \cref{eqn:flops_for_C} $h$ times.

\section{Experiments}\label{sec:experiments}

\subsection{Image Classification}\label{sec:IC}

In this section, we apply our theoretical insights from \cref{subsec:precond} to vision transformers on the ImageNet-1k dataset, a benchmark with over 1.2 million labeled images across 1,000 categories. We evaluate four representative architectures: ViT~\citep{dosovitskiy2020image}, DeiT~\citep{touvron2021training}, Swin~\citep{liu2021swin}, and XCiT~\citep{ali2021xcit}. ViT-Base (ViT-B) and DeiT-Base (DeiT-B) follow the standard transformer design, with DeiT using a data-efficient training strategy. Swin-Base (Swin-B) introduces hierarchical representations with shifted window attention, while XCiT-M24 (XCiT-M) employs Local Patch Interaction and Cross-Covariance Attention. In all cases, we compare the baseline model to a variant with preconditioned attention, where each attention matrix is multiplied by the preconditioner defined in \cref{thm:precond_attn}.

\begin{figure*}[ht!]
    \centering
    \includegraphics[width=0.48\linewidth]{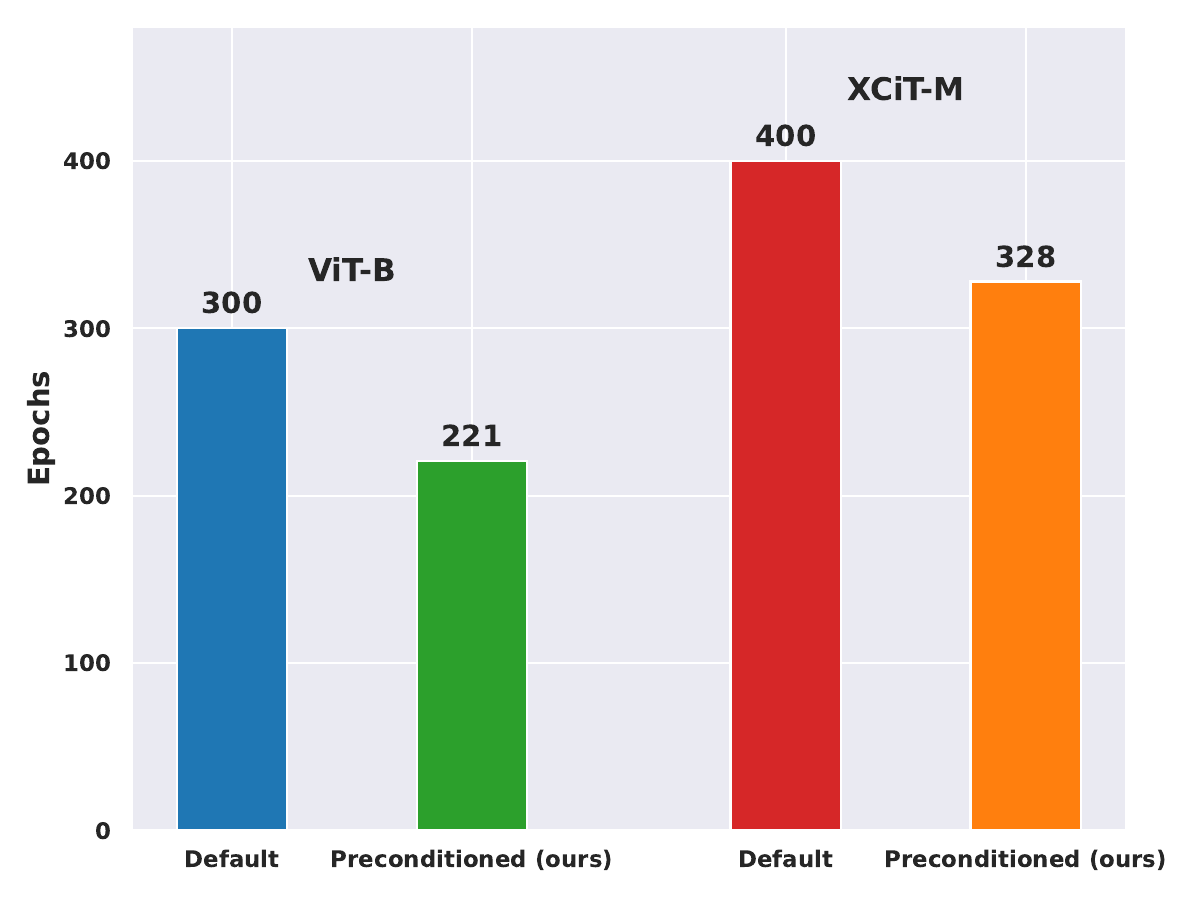}
    \hfill
    \includegraphics[width=0.48\linewidth]{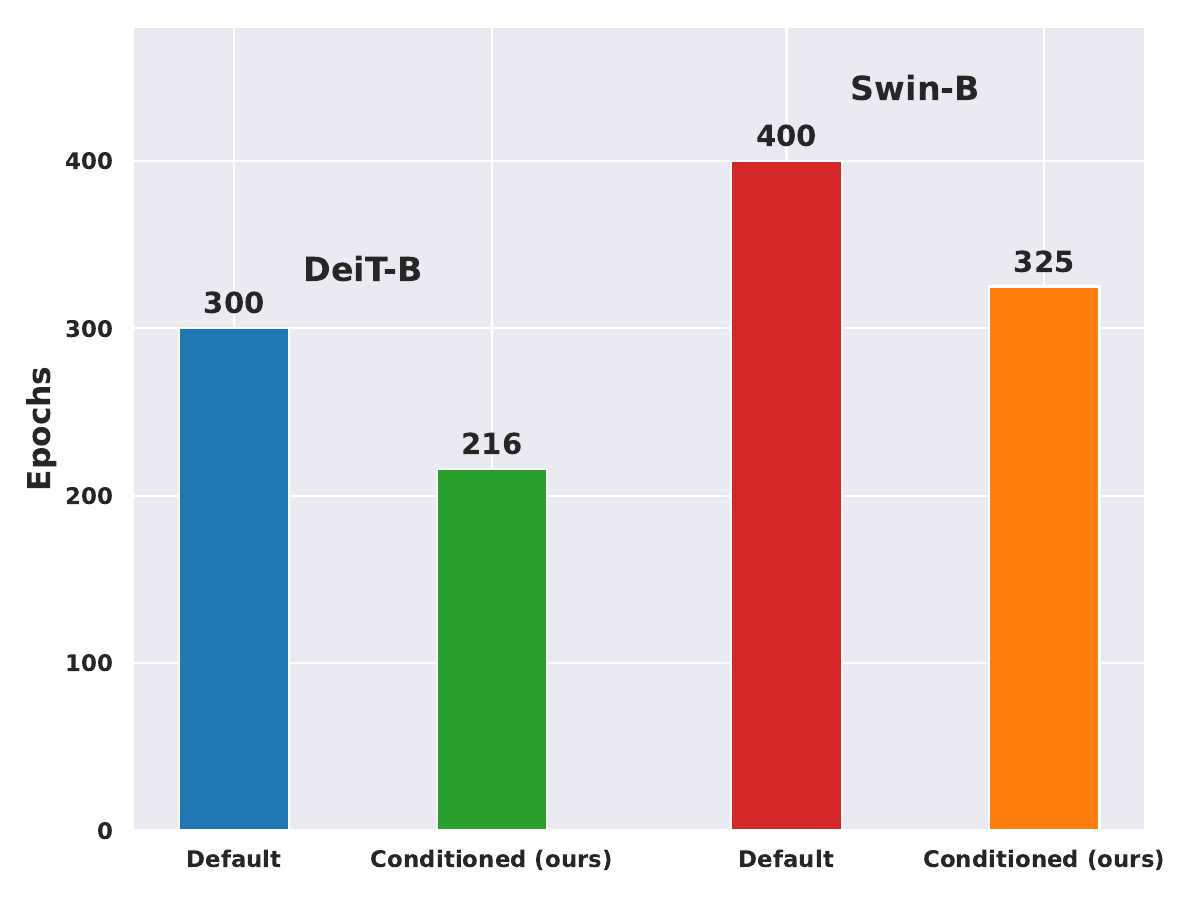}
    \caption{Total number of epochs required for each preconditioned model to reach the accuracy of the baseline model. For each ViT we see that the preconditioned model requires roughly 20-30\% less epochs.}
    \label{fig:vits_conv}
\end{figure*}

\paragraph{Results.} We compared five vision transformers using their original attention mechanisms against variants with preconditioned attention. ViT-Base~\citep{dosovitskiy2020image} and DeiT-Base~\citep{touvron2021training} employ standard self-attention, while Swin-Base~\citep{liu2021swin} and XCiT-Medium~\citep{ali2021xcit} use more specialized mechanisms. Including these latter models demonstrates that our method, though developed in the context of self-attention, extends as a drop-in replacement for modern attention designs. Full implementation details are provided in \cref{app:subsec_IC}. As shown in \cref{tab:vits}, preconditioned attention consistently improves Top-1\% accuracy on ImageNet-1k. 

\begin{figure}[ht!]
    \centering
    \includegraphics[width=1.\linewidth]
    {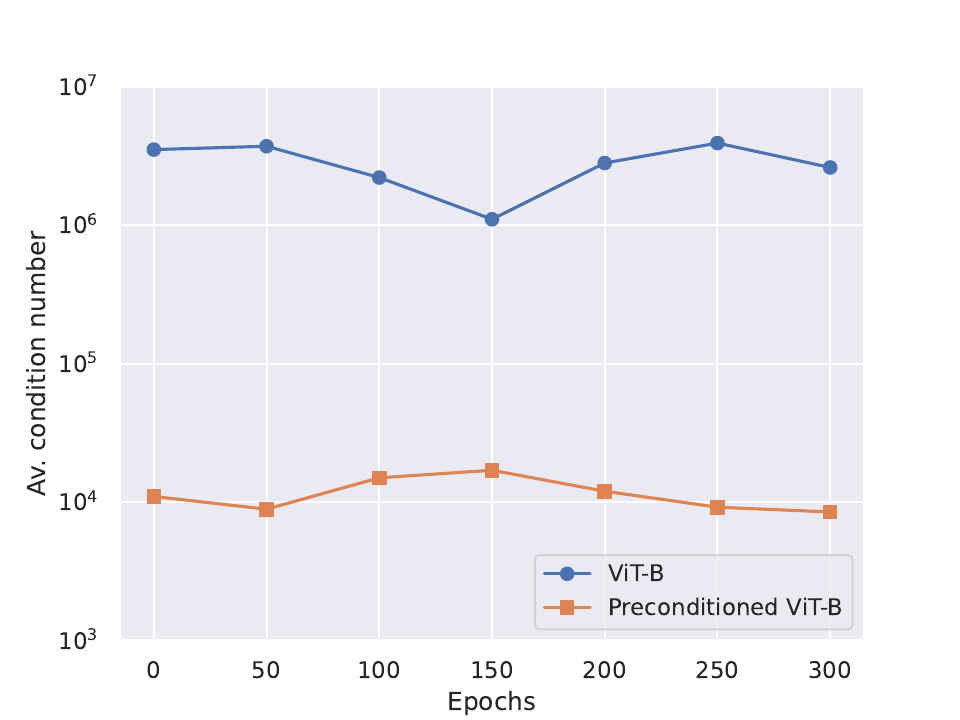}
    %\caption{Trained for 350s}
    %\end{subfigure}
    %\vspace{-2em}
    \caption{Average condition number of a ViT-B and a preconditioned ViT-B during training on the ImageNet-1k dataset.}
    \label{fig:vit_condition}
\end{figure}

\begin{table}[h!]
\caption{Training time/memory overhead for ViTs}
\centering
\resizebox{\columnwidth}{!}{%
\begin{tabular}{|c|c|c|}
\hline
\textbf{} & \textbf{Train Time}(hrs:mins) & \textbf{Memory} (GB) \\ \hline
ViT-Base      & 29:28     & 176.1     \\ 
Precond. ViT-Base      & 30:45      & 178.5      \\ \hline
DeiT-Base      & 26:16      & 173.6     \\ 
Precond. DeiT-Base      & 27:25      & 175.1      \\ \hline
Swin-Base      & 53:11      & 150.4      \\ \
Precond. Swin-Base      & 54:24      & 152.6      \\ \hline
XCiT-Med.      & 91:08     & 186.4      \\ 
Precond. XCiT-Med.      & 92:10   & 188.5      \\ \hline
\end{tabular}%
}
\label{tab:memory_vits}
\end{table}

\paragraph{Testing the theory.} We evaluated the effect of preconditioning by training both a ViT-B and a preconditioned ViT-B on the ImageNet-1k dataset. During training, we computed the condition number for each attention head in every layer, averaging first across the $12$ heads per layer and then across the $12$ layers. The results, shown in \cref{fig:vit_condition}, indicate that the preconditioned ViT-B consistently maintains a substantially lower average condition number throughout training. Furthermore, \cref{fig:vits_conv} compares convergence speed, showing that the preconditioned model reaches the baseline’s final accuracy in roughly $20$–$30\%$ fewer epochs. We note that preconditioned attention introduces additional FLOPs, as shown in \cref{subsec:cost},
leading to increased training time and memory usage as reported in \cref{tab:memory_vits}. However, the observed overhead remains modest, indicating that the performance gains come at only a minor computational cost.

\section{Transfer Learning for Detection and Segmentation}

\begin{table*}[!ht]
\caption{Performance on COCO object detection and instance segmentation (mini-val set) using Mask R-CNN with ImageNet-1k pretrained backbones. Reported metrics include bounding box AP (\(AP^b\)), AP at IoU thresholds 0.50 and 0.75 (\(AP^b_{50}\), \(AP^b_{75}\)), mask AP (\(AP^m\)), and mask AP at IoU thresholds 0.50 and 0.75 (\(AP^m_{50}\), \(AP^m_{75}\)).}
    \centering

    \setlength{\tabcolsep}{14pt}
    \begin{tabular}{c|c c c |c c c}
        \midrule
         Model & $AP^b$ & $AP^b_{50}$ & $AP^b_{75}$ & $AP^m$ & $AP^m_{50}$ & $AP^m_{75}$ \\
        \midrule
        XCiT-S12 & 44.9 & 66.1 & 48.9 & 40.1 & 63.1 & 42.8 \\
        \midrule
        Preconditioned XCiT-S12 & \textbf{45.4} & \textbf{66.4} & \textbf{49.5} & \textbf{40.5} & \textbf{63.3} & \textbf{43.2} \\
        \midrule
        XCiT-M24 & 45.7 & 66.8 & 49.6 & 40.8 & 63.6 & 43.3 \\
        \midrule
        Preconditioned XCiT-M24 & \textbf{46.0} & \textbf{67.2} & \textbf{49.9} & \textbf{41.2} & \textbf{63.9} & \textbf{44.0} \\
        \midrule
    \end{tabular}

    \label{tab:transferlearning}
\end{table*}

\begin{table*}[!ht]
  \caption{Comparison of a Nystr\"omformer with a preconditioned Nystr\"omformer on LRA benchmark. We report evaluation accuracy (\%).}
    \centering
    \setlength{\tabcolsep}{12pt}
    \begin{tabular}{c|c c c c c }
        \midrule
         Model & ListOps & Text & Retrieval & Image & Pathfinder \\
        \midrule
         Nystr\"omformer & 37.1 & 63.8 & 79.8 & 39.9 & 72.9 \\
        \midrule
        Preconditioned Nystr\"omformer & \textbf{37.9} & \textbf{64.8} & \textbf{80.7} & \textbf{40.8} & \textbf{73.8} \\
        \midrule
    \end{tabular}
    \label{tab:nystrom}
\end{table*}

\begin{figure*}[ht!]
    \centering
    \includegraphics[width=0.32\linewidth]{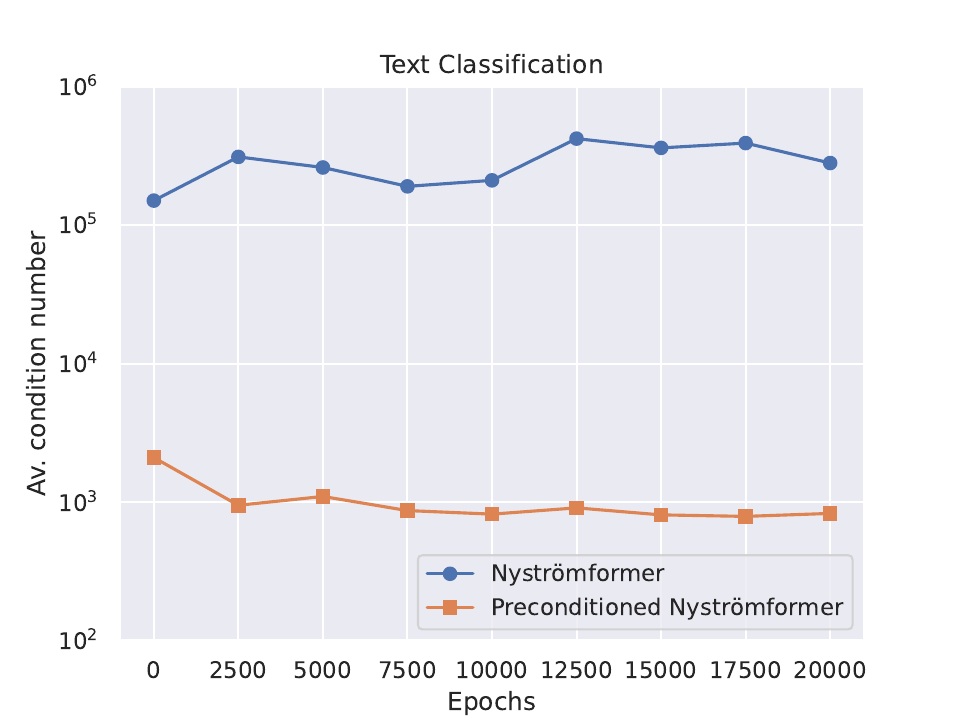}
    \hfill
    \includegraphics[width=0.32\linewidth]{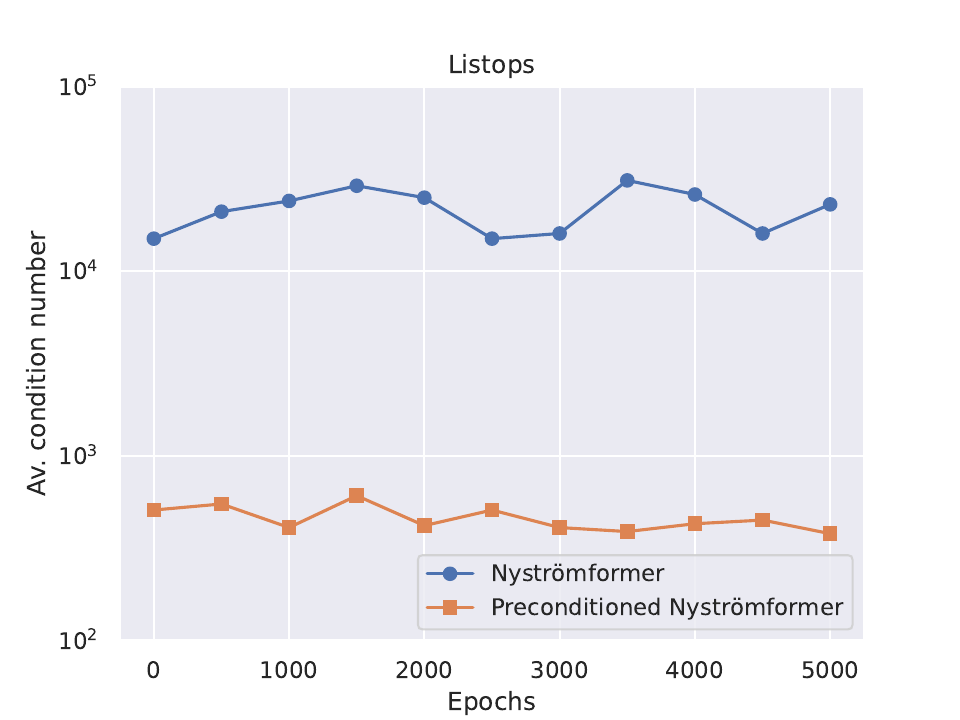}
    \hfill
    \includegraphics[width=0.32\linewidth]{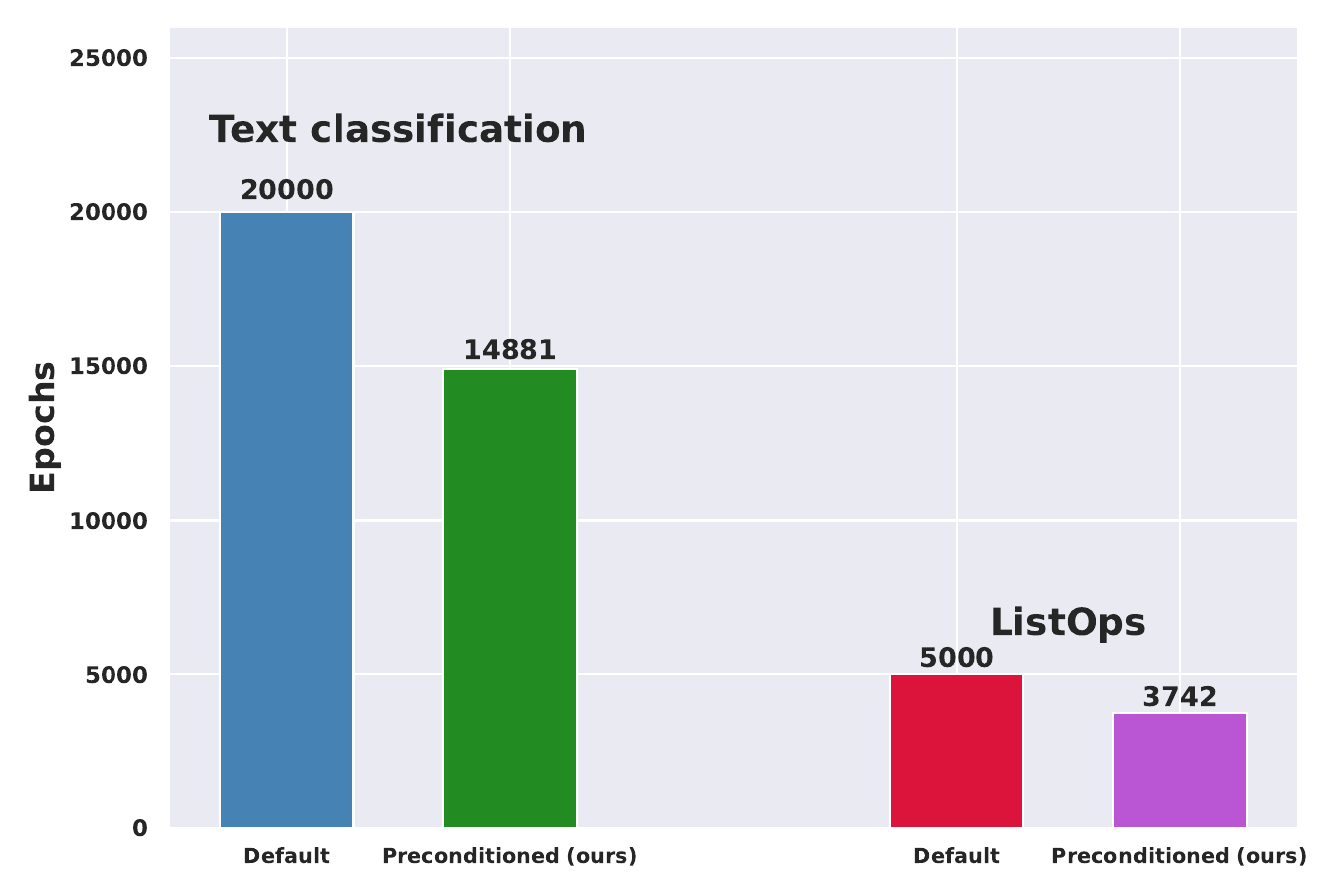}
    \caption{Conditioning analysis of the Nystr\"omformer on text classification and ListOps. The preconditioned variant achieves consistently lower condition numbers (left and middle) and converges faster, requiring fewer iterations to reach the baseline model’s final accuracy (right).}
    \label{fig:nystrom_condition}
\end{figure*}

In this section, we investigate the transfer learning ability of our models by evaluating them on object detection and instance segmentation. We fine-tune XCiT models pretrained on ImageNet-1k and assess their performance on the COCO 2017 benchmark~\citep{lin2014microsoft}, which contains 118K training images and 5K validation images across 80 object categories. The XCiT backbone is integrated into the Mask R-CNN framework~\citep{he2017mask}, equipped with a Feature Pyramid Network (FPN) to capture multi-scale representations.  

To interface XCiT with FPN, we adapt its columnar architecture by extracting intermediate features at multiple layers. For both XCiT-S12 and XCiT-M24, the native stride-16 outputs are converted into multi-scale features with strides \([4, 8, 16, 32]\). This is achieved through max pooling for downsampling and a transposed convolution for upsampling. Fine-tuning is carried out for 36 epochs with the AdamW optimizer, using a learning rate of \(10^{-4}\), weight decay of \(0.05\), and a batch size of 16. These experiments demonstrate that XCiT’s learned representations transfer effectively to detection and segmentation tasks, underscoring the flexibility of the architecture beyond image classification.  

\begin{table*}[h]
\setlength{\tabcolsep}{7pt}  % Wider column padding
\caption{Evaluation of pretrained Crammed BERT using the default baseline and a variant with preconditioned attention (as defined in \cref{alg:precond_attn}). The preconditioned model consistently outperforms the baseline across GLUE tasks.
}
\label{tab:bert_glue_results}
\centering
\begin{tabular}{lccccccccc}
    \toprule
    & MNLI & SST-2 & STSB & RTE & QNLI & QQP & MRPC & CoLA & Avg. \\
    \midrule
    Default     & 83.8 & 92.3 & 86.3 & 55.1 & 90.1 & 87.3 & 85.0 & 48.9 & 78.6 \\
    Preconditioned     & \textbf{84.9} & \textbf{92.9} & \textbf{87.5} & \textbf{55.9} & \textbf{90.9} & \textbf{88.4} & \textbf{85.9} & \textbf{50.0} & \textbf{79.6} \\
    \bottomrule
\end{tabular}
\end{table*}

\paragraph{Results.}  
We pretrained four models on the ImageNet-1k dataset: the standard XCiT-S12 and XCiT-M24, along with their counterparts incorporating a preconditioned attention layer. The latter modifies the attention mechanism by applying the diagonal preconditioner introduced in \cref{thm:precond_attn} to the final attention output before multiplication with the values matrix. The transfer learning results are summarized in \cref{tab:transferlearning}. Across both model sizes, the variants equipped with preconditioned attention consistently outperform their standard counterparts, highlighting the benefits of our proposed conditioning approach.

\subsection{Long Range Sequence Modeling}

Transformers often struggle with modeling long-range dependencies. The Long-Range Arena (LRA) benchmark~\citep{tay2020long} provides a standardized suite of five tasks—ListOps, text classification, retrieval, image classification, and Pathfinder, specifically designed to evaluate this capability. For detailed descriptions of these tasks, we refer the reader to \cite{tay2020long}. In our experiments, we use the Nystr\"omformer~\citep{xiong2021nystromformer}, which approximates attention in near-linear time, and compare its standard form against a variant enhanced with preconditioned attention.

\paragraph{Results.} We evaluated both the standard Nystr\"omformer and its preconditioned variant across all five tasks in the LRA benchmark suite. For each task, we strictly followed the training protocol of~\cite{xiong2021nystromformer}, as detailed in their public implementation~\cite{nystromformer_github}. The final accuracies are reported in \cref{tab:nystrom}, where the preconditioned Nystr\"omformer consistently outperforms the baseline across all tasks. For details on memory usage and training time overhead, see \cref{app:nystrom_sec}.

\paragraph{Testing the theory.}  
Our first experiment examines the conditioning of the attention matrix during training on the text and ListOps tasks. We compare the standard Nystr\"omformer with its preconditioned variant by tracking the condition number of each attention head. As shown in \cref{fig:nystrom_condition}, the preconditioned Nystr\"omformer consistently maintains lower condition numbers (left and middle) throughout training. The right panel of \cref{fig:nystrom_condition} reports the number of iterations required to match the final accuracy of the baseline model. In both tasks, the preconditioned variant converges more quickly, requiring fewer iterations to reach the same performance.

\subsection{Language Modeling}

Building on the theoretical insights from \cref{subsec:precond}, we evaluate preconditioned attention in the Crammed BERT language model~\citep{geiping2023cramming} trained with masked language modeling. We pretrain two variants on The Pile~\citep{gao2021pile}: a standard baseline and a model with preconditioned attention following the implementation in \cref{alg:precond_attn}. Their performance is then assessed on the GLUE benchmark~\citep{wang2018glue}.

\paragraph{Results.} \Cref{tab:bert_glue_results} reports the GLUE benchmark performance for the default Crammed BERT and the variant with preconditioned attention. The results show that preconditioned attention consistently yields higher scores across tasks, demonstrating its effectiveness.

\section{Limitations}

Our work introduces a preconditioner matrix aimed at improving the conditioning of the self-attention block in transformer layers. The design is motivated by \cref{thm:precond_attn}, which leverages the complexity measure in \cref{eqn:condition_complxity_quant} as an upper bound on the condition number. Although our experiments validate this measure as a useful proxy, it remains an indirect approach. A more explicit method for computing the condition number of the self-attention matrix could provide a stronger foundation for preconditioning and potentially lead to further performance gains.

\section{Conclusion}

We introduced a theoretical framework for analyzing the conditioning of the self-attention matrix in transformer layers and proposed a diagonal preconditioner that reduces its condition number. This yields a better-conditioned self-attention mechanism, which we termed \emph{preconditioned self-attention}. Our empirical results across image classification, long-range sequence modeling, and language modeling demonstrate that this approach consistently improves performance. Notably, preconditioned self-attention acts as a simple drop-in replacement for existing attention mechanisms, providing broad applicability and consistent gains.

%However, as shown in \cref{subsec:cost}, applying the preconditioner introduces computational overhead. For instance, \cref{tab:memory_vits} highlights the additional training time and memory requirements associated with conditioned attention. This raises an important question: can a more efficient method for computing the preconditioner be developed to reduce these overheads? We believe this is an interesting problem and plan to address it in future work.

\newpage
\bibliography{iclr2026_conference}

@article{vaswani2017attention,
  title={Attention is all you need},
  author={Vaswani, Ashish and Shazeer, Noam and Parmar, Niki and Uszkoreit, Jakob and Jones, Llion and Gomez, Aidan N and Kaiser, {\L}ukasz and Polosukhin, Illia},
  journal={Advances in neural information processing systems},
  volume={30},
  year={2017}
}

@book{nocedal1999numerical,
  title={Numerical optimization},
  author={Nocedal, Jorge and Wright, Stephen J},
  year={1999},
  publisher={Springer}
}

@inproceedings{saratchandran2025weight,
  title={Weight Conditioning for Smooth Optimization of Neural Networks},
  author={Saratchandran, Hemanth and Wang, Thomas X and Lucey, Simon},
  booktitle={European Conference on Computer Vision},
  pages={310--325},
  year={2025},
  organization={Springer}
}

@inproceedings{agarwal2021deep,
  title={A deep conditioning treatment of neural networks},
  author={Agarwal, Naman and Awasthi, Pranjal and Kale, Satyen},
  booktitle={Algorithmic Learning Theory},
  pages={249--305},
  year={2021},
  organization={PMLR}
}

@article{liu2022loss,
  title={Loss landscapes and optimization in over-parameterized non-linear systems and neural networks},
  author={Liu, Chaoyue and Zhu, Libin and Belkin, Mikhail},
  journal={Applied and Computational Harmonic Analysis},
  volume={59},
  pages={85--116},
  year={2022},
  publisher={Elsevier}
}

@article{saratchandran2024rethinking,
  title={Rethinking Attention: Polynomial Alternatives to Softmax in Transformers},
  author={Saratchandran, Hemanth and Zheng, Jianqiao and Ji, Yiping and Zhang, Wenbo and Lucey, Simon},
  journal={arXiv preprint arXiv:2410.18613},
  year={2024}
}

@article{zheng2025structured,
  title={Structured Initialization for Vision Transformers},
  author={Zheng, Jianqiao and Li, Xueqian and Saratchandran, Hemanth and Lucey, Simon},
  journal={arXiv preprint arXiv:2505.19985},
  year={2025}
}

@inproceedings{chng2025preconditioners,
  title={Preconditioners for the stochastic training of neural fields},
  author={Chng, Shin-Fang and Saratchandran, Hemanth and Lucey, Simon},
  booktitle={Proceedings of the Computer Vision and Pattern Recognition Conference},
  pages={27222--27232},
  year={2025}
}

@article{saratchandran2025leaner,
  title={Leaner transformers: More heads, less depth},
  author={Saratchandran, Hemanth and Teney, Damien and Lucey, Simon},
  journal={arXiv preprint arXiv:2505.20802},
  year={2025}
}

@inproceedings{saratchandran2025enhancing,
  title={Enhancing transformers through conditioned embedded tokens},
  author={Saratchandran, Hemanth and Lucey, Simon},
  booktitle={Proceedings of the IEEE/CVF International Conference on Computer Vision},
  pages={4786--4795},
  year={2025}
}

@article{saratchandran2026spectral,
  title={Spectral conditioning of attention improves transformer performance},
  author={Saratchandran, Hemanth and Lucey, Simon},
  journal={arXiv preprint arXiv:2603.07162},
  year={2026}
}

@article{qin2022cosformer,
  title={cosformer: Rethinking softmax in attention},
  author={Qin, Zhen and Sun, Weixuan and Deng, Hui and Li, Dongxu and Wei, Yunshen and Lv, Baohong and Yan, Junjie and Kong, Lingpeng and Zhong, Yiran},
  journal={arXiv preprint arXiv:2202.08791},
  year={2022}
}

@article{ji2025cutting,
  title={Cutting the Skip: Training Residual-Free Transformers},
  author={Ji, Yiping and Martens, James and Zheng, Jianqiao and Zhou, Ziqin and Moghadam, Peyman and Zhang, Xinyu and Saratchandran, Hemanth and Lucey, Simon},
  journal={arXiv preprint arXiv:2510.00345},
  year={2025}
}

@article{macdonald2023skip,
  title={On skip connections and normalisation layers in deep optimisation},
  author={MacDonald, Lachlan and Valmadre, Jack and Saratchandran, Hemanth and Lucey, Simon},
  journal={Advances in Neural Information Processing Systems},
  volume={36},
  pages={14705--14724},
  year={2023}
}

@article{ji2024sine,
  title={Sine activated low-rank matrices for parameter efficient learning},
  author={Ji, Yiping and Saratchandran, Hemanth and Gordon, Cameron and Zhang, Zeyu and Lucey, Simon},
  journal={arXiv preprint arXiv:2403.19243},
  year={2024}
}

@article{albert2025randlora,
  title={Randlora: Full-rank parameter-efficient fine-tuning of large models},
  author={Albert, Paul and Zhang, Frederic Z and Saratchandran, Hemanth and Rodriguez-Opazo, Cristian and Hengel, Anton van den and Abbasnejad, Ehsan},
  journal={arXiv preprint arXiv:2502.00987},
  year={2025}
}

@inproceedings{albert2025towards,
  title={Towards Higher Effective Rank in Parameter-Efficient Fine-tuning using Khatri-Rao Product},
  author={Albert, Paul and Zhang, Frederic Z and Saratchandran, Hemanth and van den Hengel, Anton and Abbasnejad, Ehsan},
  booktitle={Proceedings of the IEEE/CVF International Conference on Computer Vision},
  pages={1292--1302},
  year={2025}
}

@article{ali2021xcit,
  title={Xcit: Cross-covariance image transformers},
  author={Ali, Alaaeldin and Touvron, Hugo and Caron, Mathilde and Bojanowski, Piotr and Douze, Matthijs and Joulin, Armand and Laptev, Ivan and Neverova, Natalia and Synnaeve, Gabriel and Verbeek, Jakob and others},
  journal={Advances in neural information processing systems},
  volume={34},
  pages={20014--20027},
  year={2021}
}

@inproceedings{liu2021swin,
  title={Swin transformer: Hierarchical vision transformer using shifted windows},
  author={Liu, Ze and Lin, Yutong and Cao, Yue and Hu, Han and Wei, Yixuan and Zhang, Zheng and Lin, Stephen and Guo, Baining},
  booktitle={Proceedings of the IEEE/CVF international conference on computer vision},
  pages={10012--10022},
  year={2021}
}

@inproceedings{touvron2021training,
  title={Training data-efficient image transformers \& distillation through attention},
  author={Touvron, Hugo and Cord, Matthieu and Douze, Matthijs and Massa, Francisco and Sablayrolles, Alexandre and J{\'e}gou, Herv{\'e}},
  booktitle={International conference on machine learning},
  pages={10347--10357},
  year={2021},
  organization={PMLR}
}

@article{yuan2022volo,
  title={Volo: Vision outlooker for visual recognition},
  author={Yuan, Li and Hou, Qibin and Jiang, Zihang and Feng, Jiashi and Yan, Shuicheng},
  journal={IEEE transactions on pattern analysis and machine intelligence},
  volume={45},
  number={5},
  pages={6575--6586},
  year={2022},
  publisher={IEEE}
}

@inproceedings{ding2022davit,
  title={Davit: Dual attention vision transformers},
  author={Ding, Mingyu and Xiao, Bin and Codella, Noel and Luo, Ping and Wang, Jingdong and Yuan, Lu},
  booktitle={European conference on computer vision},
  pages={74--92},
  year={2022},
  organization={Springer}
}

@inproceedings{xiong2021nystromformer,
  title={Nystr{\"o}mformer: A nystr{\"o}m-based algorithm for approximating self-attention},
  author={Xiong, Yunyang and Zeng, Zhanpeng and Chakraborty, Rudrasis and Tan, Mingxing and Fung, Glenn and Li, Yin and Singh, Vikas},
  booktitle={Proceedings of the AAAI Conference on Artificial Intelligence},
  volume={35},
  pages={14138--14148},
  year={2021}
}

@article{guggenheimer1995simple,
  title={A simple estimate of the condition number of a linear system},
  author={Guggenheimer, Heinrich W and Edelman, Alan S and Johnson, Charles R},
  journal={The College Mathematics Journal},
  volume={26},
  number={1},
  pages={2--5},
  year={1995},
  publisher={Taylor \& Francis}
}

@article{dosovitskiy2020image,
  title={An image is worth 16x16 words},
  author={Dosovitskiy, Alexey and Beyer, Lucas and Kolesnikov, Alexander and Weissenborn, Dirk and Zhai, Xiaohua and Unterthiner, Thomas and Dehghani, Mostafa and Minderer, Matthias and Heigold, Georg and Gelly, Sylvain and others},
  journal={arXiv preprint arXiv:2010.11929},
  volume={7},
  year={2020}
}

@inproceedings{lin2014microsoft,
  title={Microsoft coco: Common objects in context},
  author={Lin, Tsung-Yi and Maire, Michael and Belongie, Serge and Hays, James and Perona, Pietro and Ramanan, Deva and Doll{\'a}r, Piotr and Zitnick, C Lawrence},
  booktitle={Computer Vision--ECCV 2014: 13th European Conference, Zurich, Switzerland, September 6-12, 2014, Proceedings, Part V 13},
  pages={740--755},
  year={2014},
  organization={Springer}
}

@inproceedings{he2017mask,
  title={Mask r-cnn},
  author={He, Kaiming and Gkioxari, Georgia and Doll{\'a}r, Piotr and Girshick, Ross},
  booktitle={Proceedings of the IEEE international conference on computer vision},
  pages={2961--2969},
  year={2017}
}

@article{tay2020long,
  title={Long range arena: A benchmark for efficient transformers},
  author={Tay, Yi and Dehghani, Mostafa and Abnar, Samira and Shen, Yikang and Bahri, Dara and Pham, Philip and Rao, Jinfeng and Yang, Liu and Ruder, Sebastian and Metzler, Donald},
  journal={arXiv preprint arXiv:2011.04006},
  year={2020}
}

@misc{nystromformer_github,
  author       = {Xiong, Yunyang and Zeng, Zhanpeng and Chakraborty, Rudrasis and Tan, Fei and Fung, Glenn and Singh, Vikas and Yuan, Xiaodong and Wang, Sungsoo Ahn and Papailiopoulos, Dimitris and Fragkiadaki, Katerina},
  title        = {GitHub repository},
  year         = {2021},
  url          = {https://github.com/mlpen/Nystromformer},
}

@misc{rw2019timm,
  author = {Ross Wightman},
  title = {PyTorch Image Models},
  year = {2019},
  publisher = {GitHub},
  journal = {GitHub repository},
  doi = {10.5281/zenodo.4414861},
  howpublished = {\url{https://github.com/rwightman/pytorch-image-models}}
}

@article{ji2025always,
  title={Always skip attention},
  author={Ji, Yiping and Saratchandran, Hemanth and Moghadam, Peyman and Lucey, Simon},
  journal={arXiv preprint arXiv:2505.01996},
  year={2025}
}

@article{jacot2018neural,
  title={Neural tangent kernel: Convergence and generalization in neural networks},
  author={Jacot, Arthur and Gabriel, Franck and Hongler, Cl{\'e}ment},
  journal={Advances in neural information processing systems},
  volume={31},
  year={2018}
}

@inproceedings{geiping2023cramming,
  title={Cramming: Training a Language Model on a single GPU in one day.},
  author={Geiping, Jonas and Goldstein, Tom},
  booktitle={International Conference on Machine Learning},
  pages={11117--11143},
  year={2023},
  organization={PMLR}
}

@article{gao2021pile,
  title={The Pile: An 800GB Dataset of Diverse Text for Language Modeling},
  author={Gao, Leo and Biderman, Stella and Black, Sid and Golding, Laurence and Hoppe, Travis and Foster, Charles and Phang, Jason and He, Horace and Thite, Aadi and Nabeshima, Eric and others},
  journal={arXiv preprint arXiv:2101.00027},
  year={2021}
}

@article{wang2018glue,
  title={GLUE: A multi-task benchmark and analysis platform for natural language understanding},
  author={Wang, Alex and Singh, Amanpreet and Michael, Julian and Hill, Felix and Levy, Omer and Bowman, Samuel R},
  journal={arXiv preprint arXiv:1804.07461},
  year={2018}
}
\bibliographystyle{iclr2026_conference}

%\begin{thebibliography}{}
%\setlength{\itemindent}{-\leftmargin}
%\makeatletter\renewcommand{\@biblabel}[1]{}\makeatother
%\bibitem{} J.~Alspector, B.~Gupta, and R.~B.~Allen (1989).
%    \newblock Performance of a stochastic learning microchip.
%    \newblock In D. S. Touretzky (ed.),
%    \textit{Advances in Neural Information Processing Systems 1}, 748--760.
%    San Mateo, Calif.: Morgan Kaufmann.

%\bibitem{} F.~Rosenblatt (1962).
%    \newblock \textit{Principles of Neurodynamics.}
%    \newblock Washington, D.C.: Spartan Books.

%\bibitem{} G.~Tesauro (1989).
%    \newblock Neurogammon wins computer Olympiad.
%    \newblock \textit{Neural Computation} \textbf{1}(3):321--323.
%\end{thebibliography}

%%%%%%%%%%%%%%%%%%%%%%%%%%%%%%%%%%%%%%%%%%%%%%%%%%%%%%%%%%%%

\clearpage
\appendix
\thispagestyle{empty}

% Supplementary material: To improve readability, you must use a single-column format for the supplementary material.
\onecolumn
\aistatstitle{Appendix}

\section{Theoretical Analysis}\label{app:theory}

In this section we give the proofs of \cref{thm:condition_self_attn} and \cref{thm:precond_attn}.

We will start by giving the proof of \cref{thm:condition_self_attn}. In order to do this we will need a lemma about the $\mathbf{softmax}$ activation function.

\begin{lemma}\label{lem:softmax_frob}
    Let $A \in \R^{n \times d}$ be an arbitrary $n\times d$ matrix. Let $\mathbf{softmax}(A)$ denote the matrix given by applying $\mathbf{softmax}$ row wise. Then we have that
    \begin{equation}
        \vert\vert \mathbf{softmax}(A)\vert\vert_F \leq \sqrt{n}.
    \end{equation}
\end{lemma}
\begin{proof}
For ease of notation let $B = \mathbf{softmax}(A)$ which is an $n \times d$ matrix. Let $B_{ij}$ denote the entry of $B$ in the i-th row and jth column. We then have 
\begin{align}
    \vert\vert B\vert\vert_F^2 &= |B_{11}|^2 + \cdots + |B_{1d}|^2 \\
    &\hspace{0.5cm} + \\
    &\hspace{0.5cm} \vdots \\
    &\hspace{0.5cm} + \\
    &\hspace{0.5cm}|B_{n1}|^2 + \cdots + |B_{nd}|^2. 
\end{align}
Then observe that for each $1 \leq i \leq n$ we have that
\begin{align}
|B_{i1}|^2 + \cdots + |B_{id}|^2 &\leq  (|B_{i1}| + \cdots + |B_{id}|)^2 \\
&= 1 
\end{align}
where the equality comes from the fact that we applied 
$\mathbf{softmax}$ row wise. Doing this for each of the $n$ rows of $B$ we get
\begin{equation}
   \vert\vert B\vert\vert_F^2 \leq n 
\end{equation}
which gives the desired result.
\end{proof}

\begin{proof}[Proof of \cref{thm:condition_self_attn}]
We start by using Guggenheimer et al. result from \cref{thm:guggen} and compute the term
\begin{equation}
\bigg{(}\frac{\vert\vert \mathbf{softmax}(qk^T)v\vert\vert_F}{\sqrt{k}} \bigg{)}^k.
\end{equation}
We observe that $\vert\vert \mathbf{softmax}(qk^T)v\vert\vert_F 
\leq \vert\vert \mathbf{softmax}(qk^T)\vert\vert_F\cdot  \vert\vert v\vert\vert_F$ and by \cref{lem:softmax_frob} we have that
$\vert\vert \mathbf{softmax}(qk^T)\vert\vert_F \leq \sqrt{n}$. This implies that 
\begin{equation}
\bigg{(}\frac{\vert\vert \mathbf{softmax}(qk^T)v\vert\vert_F}{\sqrt{k}} \bigg{)}^k \leq 
\bigg{(}\frac{\sqrt{n}\vert\vert v\vert\vert_F}{\sqrt{k}} \bigg{)}^k
\end{equation}
which proves the theorem.
\end{proof}

We can now give the proof of \cref{thm:precond_attn}. In order to do this we will need the following lemma.

\begin{lemma}\label{lem:sum_posi}
Let $a_1, \ldots, a_n$ be $n$ non-negative numbers such that
$a_1 \leq \cdots \leq a_n$ and such that
\begin{equation}
    a_1 + \cdots + a_n = n.
\end{equation}
Then for any $1 \leq k \leq n$ we must have that
\begin{equation}
a_1 + \cdots + a_k \leq k.
\end{equation}
\end{lemma}
\begin{proof}
Suppose that $a_1 + \cdots + a_k > k$. We then have that
\begin{align}
    n = a_1 + \cdots + a_n > k + a_{n+1} + \cdots + a_n 
\end{align}
which implies
\begin{equation}\label{eqn:sum_contra}
    a_{k+1} + \cdots + a_n < n-k
\end{equation}
However, since we are assuming $a_1 + \cdots + a_k > k$ we must have that at least one of the $a_i > 1$ for $1 \leq i \leq k$. This then gives a contradiction as each $a_{k+j} \geq a_i$ for 
$1 \leq i \leq k$ and $1 \leq j \leq n-k$ which then implies
\begin{equation}
    a_{k+1} + \cdots + a_n \geq (n-k)a_i \geq n-k
\end{equation}
which contradicts \cref{eqn:sum_contra}. Therefore, we must have that
$a_1 + \cdots + a_k \leq k$ which proves the result.
\end{proof}

\begin{proof}[Proof of \cref{thm:precond_attn}]
Our goal is to prove that 
$\mu(C\cdot \mathbf{softmax}(qk^T)v) \leq 
\mu(\mathbf{softmax}(qk^T)v)$ where $\mu$ is defined by 
\cref{eqn:condition_complxity_quant}. For this proof, let 
$A := \mathbf{softmax}(qk^T)v$. It suffices for us to prove that
\begin{equation}
    \frac{\mu(A)}{\mu(CA)} \geq 1.
\end{equation}

Writing out the definitions of $\mu(A)$ and $\mu(C)$ we want to show that 
\begin{equation}
\bigg{(}\frac{2}{\sigma_1(A)\cdots\sigma_k(A)}\bigg{(} 
\frac{\vert\vert A\vert\vert_F}{\sqrt{k}}
\bigg{)}^k\bigg{)}\Bigg{/} 
\bigg{(}\frac{2}{\sigma_1(CA)\cdots\sigma_k(CA)}\bigg{(} 
\frac{\vert\vert CA\vert\vert_F}{\sqrt{k}}
\bigg{)}^k\bigg{)} \geq 1.
\end{equation}
Observe that by since $C$ is a diagonal $n \times n$ matrix whose i-th diagonal element is $1/||A_{i,}||_2$ we have that 
$\vert\vert CA\vert\vert_F = \sqrt{n}$. Furthermore, observe that for any positive number $\lambda > 0$ we have that
\begin{equation}\label{eqn:scale_inv}
    \mu(\lambda A) = \mu(A)
\end{equation}
which follows from the fact that 
$||\lambda A||_F = \lambda ||A||_F$ and that for each 
$1 \leq i \leq k$ we have $\sigma_i(\lambda A) = 
\lambda\sigma_i(A)$. This implies we can assume 
$\vert\vert A\vert\vert_F = a\cdot\sqrt{n}$ for some fixed $a > 0$. For if not we can just scale $A$ appropriately and work with the scaled $A$ due to \cref{eqn:scale_inv}. We then find that
\begin{equation}
\bigg{(}\frac{2}{\sigma_1(A)\cdots\sigma_k(A)}\bigg{(} 
\frac{\vert\vert A\vert\vert_F}{\sqrt{k}}
\bigg{)}^k\bigg{)}\Bigg{/} 
\bigg{(}\frac{2}{\sigma_1(CA)\cdots\sigma_k(CA)}\bigg{(} 
\frac{\vert\vert CA\vert\vert_F}{\sqrt{k}}
\bigg{)}^k\bigg{)} = 
\frac{\sigma_1(CA)\cdots \sigma_k(CA)\cdot a^k}{\sigma_1(A)\cdots \sigma_k(A)}.
\end{equation}
We then use the fact that for any matrices we have
\begin{equation}
    \sigma_i(CA) \geq \sigma_{\min}(C)\sigma_i(A)
\end{equation}
where $\sigma_{\min}(C)$ denotes the smallest singular value of $C$. Note that by definition $C$ must have rank $n$ as it is a $n \times n$ diagonal matrix built out of the inverses of the 2-norms of $A$ and hence no diagonal entry of $C$ can be zero.

We then obtain
\begin{equation}
\frac{\sigma_1(CA)\cdots \sigma_k(CA)\cdot a^k}{\sigma_1(A)\cdots \sigma_k(A)} \geq \sigma_{\min}(C)^k\cdot a^k.    
\end{equation}
It now follows that by choosing $a > 0$, in particular choosing 
$a \geq \frac{1}{\sigma_{\min}(C)}$,
large enough we can obtain 
\begin{equation}
    \sigma_{\min}(C)^k\cdot a^k \geq 1
\end{equation}
which completes the proof.

\end{proof}

\section{Experiments}

Each experiment in \cref{sec:experiments} was repeated five times with different random seeds, and the reported results correspond to the mean across these runs.

\subsubsection{Image Classification}\label{app:subsec_IC}

\paragraph{Hardware and Implementation.} Each image classification experiment was run five time with five different random seeds and the mean results given in \cref{sec:IC}. The hyperparameters and training methodology for each ViT followed the original papers:
ViT~\citep{dosovitskiy2020image}, DeiT~\citep{touvron2021training}, Swin~\citep{liu2021swin}, and XCiT~\citep{ali2021xcit}. All the experiments were carried out on Nvidia A100 GPUs.
The implementation of the ViTs were all done using the Timm code base \cite{rw2019timm}

\subsubsection{Object detection and instance segmentation}\label{app:object}

\paragraph{Hardware and Implementation.} All the experiments were carried out on Nvidia A100 GPUs. The implementation followed  \cite{he2017mask} using the GitHub: 
\url{https://github.com/matterport/Mask_RCNN}

\subsubsection{Nystr\"omformer on LRA benchmark}\label{app:nystrom_sec}

\paragraph{Hardware and Implementation:} All the experiments were carried out on Nvidia A100 GPUs following the implementation and hyperparameter settings given in \cite{nystromformer_github}. \cref{tab:nyst_mem} shows the memory overhead for using a preconditioned  Nystr\"omformer and \cref{tab:nyst_time} shows the total training time on the whole LRA benchmark.

\begin{table}[h]
    \caption{Memory used for training Nystr\"omformer and preconditioned Nystr\"omformer on LRA benchmark.}
    \centering
    \begin{tabular}{|c|c|c|c|c|c|}
        \hline
        {} & Listops (GB) & Text Class. (GB) & Retrieval (GB) & Image Class. (GB) & Pathfinder (GB) \\
        \hline
        Original & 5.6 & 3.6 & 13.8 & 9.6 & 9.6 \\
        \hline
        Preconditioned & 5.8 & 3.6 & 15.8 & 9.6 & 9.6 \\
        \hline
    \end{tabular}
    \label{tab:nyst_mem}
\end{table}

\begin{table}[h]
 \caption{Total training time for Nystr\"omformer and preconditioned Nystr\"omformer on LRA benchmark.}
    \centering
    \begin{tabular}{|c|c|}
        \hline
         & Total train time on LRA (hrs:mins) \\
        \hline
        Nystr\"omformer & 8:30 \\
        \hline
        preconditioned Nystr\"omformer & 8:37 \\
        \hline
    \end{tabular}
    \label{tab:nyst_time}
\end{table}

\subsection{Language modeling}\label{app:language}

\paragraph{Hardware and Implementation:} All the experiments were carried out on a Nvidia A6000 GPU following the implementation and hyperparameter settings given in 
\cite{geiping2023cramming}.

\end{document}